\documentclass{llncs}
\usepackage{amsmath}
\usepackage{times}
\usepackage{indentfirst}
\usepackage{graphicx}
\usepackage{amssymb}

\newtheorem{faulse assertion}{faulse assertion}

\begin{document}

\title{A necessary and sufficient condition for two relations to induce the same definable set family}         

\author{Hua Yao, William Zhu\thanks{Corresponding author.
E-mail: williamfengzhu@gmail.com(William Zhu)} }
\institute{Lab of Granular Computing,\\
Minnan Normal University, Zhangzhou 363000, China}



\date{\today}          
\maketitle

\begin{abstract}
In Pawlak rough sets, the structure of the definable set families is simple and clear, but in generalizing rough sets, the structure of the definable set families is a bit more complex. There has been much research work focusing on this topic. However, as a fundamental issue in relation based rough sets, under what condition two relations induce the same definable set family has not been discussed. In this paper, based on the concept of the closure of relations, we present a necessary and sufficient condition for two relations to induce the same definable set family.
\end{abstract}

\section{Introduction}
Rough set theory has been proposed by Pawlak~\cite{Pawlak82Rough,Pawlak91Rough} as a tool to conceptualize, organize and analyze various
types of data in data mining. This method is especially useful for dealing with uncertain and vague knowledge
in information systems.
In theory, rough sets have been connected with
fuzzy sets~\cite{DuboisPrade90Rough,KazanciYamakDavvaz08TheLower,WuLeungMi05OnCharacterizations,Yao98OnGeneralizing},
lattices~\cite{Dai05Logic,EstajiHooshmandaslDavvaz12Roughappliedtolattice,Liu08Generalized,WangZhu13Quantitative}, hyperstructure theory~\cite{YamakKazanciDavvaz11Softhyperstructure},
matroids~\cite{HuangZhu12Geometriclattice,LiuZhu12Matroidal,TangSheZhu12matroidal,WangZhuZhuMin12matroidalstructure},
topology~\cite{Kondo05OnTheStructure,LashinKozaeKhadraMedhat05Rough,Zhu07Topological}, and so on.

Rough set theory is built on equivalence relations or partitions,
but equivalence relations and partitions are too restrictive for many applications.
To address this issue, several meaningful extensions of Pawlak rough sets have been proposed.
On one hand, Zakowski~\cite{Zakowski83Approximations} has used coverings to establish covering based rough set theory.
Many scholars~\cite{BonikowskiBryniarskiWybraniecSkardowska98Extensions,Bryniarski89ACalculus,ChenZhangYeungTsang06Rough,Pomykala87Approximation,Pomykala88Ondefinability,ZhuWang03Reduction} have done deep researches on this theory.
Recently, covering based rough set theory gained some new development~\cite{MuratDiker2012Textures,DuHuZhuMa11Rule,ZhanhongShiZengtaiGong2010Thefurther,TianYangQingguoLiBileiZhou2013Relatedfamily,YiyuYaoBingxueYao2012Coveringbased,YanlanZhangMaokangLuo2013Relationshipsbetweencovering}. On the other hand, Pawlak rough set theory has been extended to similarity relation based rough sets \cite{SlowinskiVanderpooten00AGeneralized}, tolerance relation based rough
sets \cite{SkowronStepaniuk96tolerance} and arbitrary binary relation based rough sets \cite{LiuZhu08TheAlgebraic,Yao98Constructive,Yao98OnGeneralizing,Zhu07Generalized}.

A definable set is a fundamental concept in various types of rough sets. Many studies have been done on this topic. Since the structure of the definable set families is simple and clear in Pawlak rough sets, these studies mainly focused on generalizing rough sets. Yang and Xu \cite{yang2009algebraic} studied the definable set families of relation based rough sets from algebraic aspects. In \cite{ge2011definable}, Ge and Li investigated definable sets of ten types of covering based rough sets. Pei \cite{pei2007definable} investigated the mathematical structure of the set of definable concepts in several generalized
rough set models and discussed the relationship between different rough set models. In \cite{ali2012some}, Ali et al. investigated the topological structures associated with definable sets in the generalized approximation space $(X,Y,T)$. Liu and Zhu \cite{LiuZhu08TheAlgebraic} presented the necessary and sufficient condition for definable set families to be nonempty and extended the concept of definable set. However, as a fundamental issue in relation based rough sets, under what condition two relations induce the same definable set family has not been discussed.

In this paper, based on the concept of the closure of relations, we present a necessary and sufficient condition for two relations to induce the same definable set family. First, we introduce some definitions and results of relation closures. Secondly, we study further some properties of definable sets. Thirdly, we simplify the expression of equivalent closures under certain conditions and prove that serial relations satisfy these conditions. Finally, we prove that a serial relation and its equivalent closure induce the same definable set family, and based on this, we present a necessary and sufficient condition for two relations to induce the same definable set family.

The remainder of this paper is organized as follows.
In Section~\ref{S:Preliminaries}, we review the relevant concepts and introduce some existing results, which include relations, relation based rough sets and relation closures.
In Section~\ref{The fundamental properties of definable sets}, we investigate some fundamental properties of definable sets. In Section~\ref{Simplification of equivalent closures under certain conditions}, we simplify the expression of equivalent closures under certain conditions.
In Section~\ref{A necessary and sufficient condition for two relations to induce the same definable set family}, we present a necessary and sufficient condition for two relations to induce the same definable set family. Section~\ref{S:Conclusions} concludes this paper.

\section{Preliminaries}
\label{S:Preliminaries}

In this section, we recall some basic concepts of relations, relation based rough sets and closures of relations. In this paper, we denote $\cup_{X\in S}X$ by $\cup S$, where $S$ is a set family. The fact that $A\subseteq B$ and $A\neq B$ is denoted by $A\subset B$. We denote the set of positive integers by $N^{+}$.

\subsection{Relation and relation based rough sets}
\label{Relation and relation based rough sets}

Relations, especially binary relations, are a basic concept in set theory. They play an important role in rough set theory
as well.

\begin{definition}(Relation)
\label{definitionA3}
Let $U$ be a set. Any $R\subseteq U\times U$ is called a binary relation on $U$.
If $(x,y)\in R$, we say $x$ has relation $R$ with $y$, and denote this relationship as $xRy$.
\end{definition}

Throughout this paper, a binary relation is simply called
a relation. On the basis of relations, we introduce the concepts of successor neighborhood and predecessor neighborhood.

\begin{definition}(Successor neighborhood and predecessor neighborhood)
\label{definition2}
Let $R$ be a relation on $U$ and $x\in U$. The successor neighborhood and predecessor neighborhood of $x$ are defined as $S_{R}(x)=\{y|xRy\}$ and $P_{R}(x)=\{y|yRx\}$, respectively.
\end{definition}

It is obvious $y\in S_{R}(x)\Leftrightarrow x\in P_{R}(y)$. Below we introduce some special relations.

\begin{definition}(Reflective, symmetric, transitive and serial relation)
\label{definitionA4}
Let $R$ be a relation on $U$. If for any $x\in U$, $xRx$, we say $R$ is reflective. If for any $x,y\in U$, $xRy$ implies $yRx$, we say $R$ is symmetric. If for any $x,y,z\in U$, $xRy$ and $yRz$ imply $xRz$, we say $R$ is transitive. If for any $x\in U$, there exists some $y\in U$ such that $xRy$, we say $R$ is serial.
\end{definition}

If $R$ is serial, we have that for any $x\in U$, it follows that $S_{R}(x)\neq\emptyset$.
Among various types of relations, there is an important type of relations called equivalent relations.

\begin{definition}(Equivalence relation \cite{GengQuWang2002discretemathematics})
\label{definitionA5}
Let $R\subseteq A\times A$ and $A\neq\emptyset$. If $R$ is reflective, symmetric and transitive, we say $R$ is an equivalence relation on $A$.
\end{definition}

Based on the concept of equivalence relation, we introduce the concept of equivalence class.

\begin{definition}(Equivalence class \cite{GengQuWang2002discretemathematics})
\label{definitionA6}
Let $R$ be an equivalence relation on a nonempty set $A$. We denote $S_{R}(x)$ as $[x]_{R}$ and call it the equivalence class of $x$ with respect of $R$.
\end{definition}

Given an equivalence relation, we can define a set family named quotient set.

\begin{definition}(Quotient set \cite{GengQuWang2002discretemathematics})
\label{definitionA7}
Let $R$ be an equivalence relation on a nonempty set $A$. We define the quotient set of $A$ with respect of $R$ as $A/R=\{[x]_{R}|x\in A\}$.
\end{definition}

There exists a concept called partition, which is closely related to the concept of equivalence relation.

\begin{definition}(Partition \cite{GengQuWang2002discretemathematics})
\label{definitionA2}
Let $A$ be a nonempty set and $P$ be a family of subsets of $A$.
$P$ is called a partition on $A$ if the following conditions hold: $(1)$ $\emptyset\notin P$; $(2)$
$\cup P=U$; $(3)$ for any $K,L\in P$, if $K\neq L$, $K\cap L=\emptyset$.
\end{definition}

The following theorem presents the relationship between partitions and equivalence relations.

\begin{theorem}(\cite{GengQuWang2002discretemathematics})
\label{theoremA7}
Let $A$ be a nonempty set. Then\\
$(1)$ If $R$ is an equivalence relation on $A$, $A/R$ is a partition on $A$;\\
$(2)$ If $P$ is a partition on $A$, $\{(x,y)|\exists K\in P(\{x,y\}\subseteq K)\}$ is an equivalence relation on $A$.
\end{theorem}

This theorem indicates that there is an one-to-one mapping between all the equivalence relations on a nonempty set $A$ and all the partitions on $A$. Pawlak rough sets have been extended to various types of generalizing rough sets. This paper studies relation based rough sets.

\begin{definition}(Rough set based on a relation~\cite{Yao98Constructive})
\label{definition3}
Suppose $R$ is a relation on a universe $U$. A pair of
approximation operators, $\underline{R}$, $\overline{R}$: $P(U)\rightarrow P(U)$, are defined by
\begin{center}
$\underline{R}(X)=\{x|S_{R}(x)\subseteq X\}$ and $\overline{R}(X)=\{x|S_{R}(x)\cap X\neq\emptyset\}$.
\end{center}
They are called the lower approximation operation and the upper approximation
operation of $X$, respectively.
\end{definition}

It is obvious $\underline{R}(X)=\{x|\forall y(xRy\rightarrow y\in X)\}$ and $\overline{R}(X)=\{x|\exists y(xRy\wedge y\in X)\}$. When $R$ is an equivalence relation on $U$, the above two approximation operations are called Pawlak approximation operations. The lower approximation operators and the upper approximation
operators have the following properties.

\begin{proposition}(\cite{Yao98Constructive})
\label{proposition4}
(1) $\overline{R}(X)=-\underline{R}(-X)\}$; (2) $\underline{R}(X)=-\overline{R}(-X)\}$.
\end{proposition}

When $R$ is serial, the lower approximation operators and the upper approximation operators have the following property.

\begin{proposition}(\cite{Yao98Constructive})
\label{proposition3A}
Let R be a serial relation on $U$. For any $X\subseteq U$, it follows that $\underline{R}(X)\subseteq\overline{R}(X)$.
\end{proposition}

As an important concept in rough set theory, definable sets have been studied widely. Below we introduce its definition.

\begin{definition}(Definable set family \cite{pei2007definable})
\label{definition5}
Let $R$ be a relation on $U$ and $X\subseteq U$. If $\underline{R}(X)=X$, we call $X$ an inner definable set. If $\overline{R}(X)=X$, we call $X$ an outer definable set. If $X$ is both inner and outer definable set, we call $X$ a definable set. We denote the family of all the inner definable sets, outer definable sets and definable sets of $U$ induced by $R$ as $I(U,R)$, $O(U,R)$ and $D(U,R)$, respectively.
\end{definition}

It is obvious $D(U,R)=I(U,R)\cap O(U,R)$.

\subsection{The closure of relations}
\label{The closure of relations}

For presenting the computational formulae of closure operators, we need to introduce some concepts and notations.

\begin{definition}
\label{definition25}
Let $A$ be a set. We denote $\{(x,x)|x\in A\}$ as $I_{A}$.
\end{definition}

For any reflective relation $R$ on $A$, it is obvious $I_{A}\subseteq R$. Below we introduce the converse of a relation.

\begin{definition}
\label{definition26}
Let $R$ be a relation. We define the converse of $R$ as $R^{-1}=\{(x,y)|(y,x)\in R\}$.
\end{definition}

For any symmetric relation $R$ on $A$, it is obvious $R=R^{-1}$. Conversely, if $R=R^{-1}$, $R$ is symmetric.
Below we introduce the compound of two relations.

\begin{definition}(\cite{GengQuWang2002discretemathematics})
\label{definition27}
Let $F,G$ be two relations. We define $F\circ G$ as $F\circ G=\{(x,y)|\exists z((x,\\
z)\in G\wedge(z,y)\in F)\}$.
\end{definition}

The compound of relations satisfies the associative law.

\begin{proposition}(\cite{GengQuWang2002discretemathematics})
\label{proposition27A1}
Let $R_{1},R_{2},R_{3}$ be three relations. Then $(R_{1}\circ R_{2})\circ R_{3}=R_{1}\circ(R_{2}\circ R_{3})$.
\end{proposition}

Since the associative law of the compound of relations holds, we can define the power of a relation.

\begin{definition}(\cite{GengQuWang2002discretemathematics})
\label{definition28}
Let $R\subseteq A\times A$ and $n$ be a natural number. We denote the nth power of $R$ as $R^{n}$, where
(1) $R^{0}=I_{A}$; (2) $R^{n+1}=R^{n}\circ R$.
\end{definition}

In order to turn an arbitrary relation to a reflective or symmetric or transitive relation, we introduce the concept of the closure of a relation.

\begin{definition}(Reflective (symmetric or transitive) closure \cite{GengQuWang2002discretemathematics})
\label{definition24}
Let $A\neq\emptyset$ and $R\subseteq A\times A$. $R^{\prime}$ is called the reflective (symmetric or transitive) closure of $R$ iff $R^{\prime}$ satisfies the following three conditions:\\
(1) $R^{\prime}$ is reflective (symmetric or transitive);\\
(2) $R\subseteq R^{\prime}$;\\
(3) For any reflective (symmetric or transitive) relation $R^{\prime\prime}$ on $A$, if $R\subseteq R^{\prime\prime}$, $R^{\prime}\subseteq R^{\prime\prime}$.

We denote the reflective, symmetric and transitive closure of $R$ as $r(R)$, $s(R)$ and $t(R)$, respectively.
\end{definition}

The following theorem presents the computational formulae of the above three closures.

\begin{theorem}(\cite{GengQuWang2002discretemathematics})
\label{theorem29}
Let $R\subseteq A\times A$ and $A\neq \emptyset$. Then\\
(1) $r(R)=R\cup I_{A}$; (2) $s(R)=R\cup R^{-1}$; (3) $t(R)=R\cup R^{2}\cup\cdots$.
\end{theorem}

The following proposition indicates that mixing the above three closure operations, we can obtain only two possibly different relations.

\begin{proposition}(\cite{GengQuWang2002discretemathematics})
\label{proposition8B7}
Let $A\neq\emptyset$ and $R\subseteq A\times A$. Then $rts(R)=trs(R)=tsr(R)$ and $rst(R)=str(R)=srt(R)$.
\end{proposition}

The following example indicates that $rst(R)$ may not be an equivalence relation.

\begin{example}
\label{example39B}
Let $A=\{1,2,3,4\}$ and $R=\{(1,2),(2,3),(1,4)\}$. Then $rst(R)=I_{A}\cup\{(1,2),(2,3),(1,3),(1,4),(2,1),(3,2),(3,1),(4,1)\}$. It is obvious $I_{A}\cup\{(1,2),(2,3),\\
(1,3),(1,4),(2,1),(3,2),(3,1),(4,1)\}$ is not transitive. Thus $rst(R)$ is not an equivalence relation.
\end{example}

However, $rts(R)$ has the following properties.

\begin{proposition}(\cite{GengQuWang2002discretemathematics})
\label{proposition8B8}
Let $A\neq\emptyset$ and $R\subseteq A\times A$. Then $rts(R)$ satisfies the following properties:\\
(1) $rts(R)$ is an equivalence relation;\\
(2) $R\subseteq rts(R)$;\\
(3) For any equivalence relation $R^{\prime\prime}$ on $A$, if $R\subseteq R^{\prime\prime}$, $rts(R)\subseteq R^{\prime\prime}$.
\end{proposition}

Based on the above proposition, we introduce the concept of equivalent closure.

\begin{definition}(Equivalent closure \cite{GengQuWang2002discretemathematics})
\label{definition35}
Let $A\neq\emptyset$ and $R\subseteq A\times A$. We define the equivalent closure of $R$ as $e(R)=rts(R)$.
\end{definition}

\section{The fundamental properties of definable sets}
\label{The fundamental properties of definable sets}

In this section, on the basis of some existing results, we investigate further some fundamental properties of definable sets.
Liu and Zhu \cite{LiuZhu08TheAlgebraic} gave the following proposition.

\begin{theorem}(\cite{LiuZhu08TheAlgebraic})
\label{theorem32A}
Let $R$ be a relation on $U$. Then $D(U,R)\neq\emptyset$ iff $R$ is serial.
\end{theorem}

Based on the above theorem, in order to study $D(U,R)$, in most cases we first assume $R$ is serial. In fact, we can still obtain the following proposition.

\begin{proposition}
\label{proposition44}
Let $R$ be a relation on $U$. If $I(U,R)=O(U,R)$, $R$ is serial.
\end{proposition}

\begin{proof}
We use the proof by contradiction. Suppose $R$ is not serial. It is obvious $U\in I(U,R)-O(U,R)$. Thus $I(U,R)\neq O(U,R)$. It is contradictory.
$\Box$
\end{proof}

For the simplicity of the description of the following propositions, we define several new notations.

\begin{definition}
\label{definition8}
Let $R$ be a relation on $U$ and $X\subseteq U$. We define $R_{R}(X)$, $P_{R}(X)$ and $V_{R}(X)$ as $S_{R}(X)=\cup_{x\in X}S_{R}(x)$, $P_{R}(X)=\cup_{x\in X}P_{R}(x)$ and $V_{R}(X)=R_{R}(X)\cup P_{R}(X)$, respectively.
\end{definition}

In \cite{yang2009algebraic}, Yang and Xu gave the following proposition.

\begin{proposition}(\cite{yang2009algebraic})
\label{proposition8A6}
Let $R$ be a relation on $U$ and $X\subseteq U$. Then $X\subseteq\underline{R}(X)$ iff $S_{R}(X)\subseteq X$; $\overline{R}(X)\subseteq X$ iff $P_{R}(X)\subseteq X$;
$\overline{R}(X)\subseteq X\subseteq\underline{R}(X)$ iff $V_{R}(X)\subseteq X$.
\end{proposition}

Based on Propositions \ref{proposition3A} and \ref{proposition8A6}, we have the following proposition.

\begin{proposition}
\label{proposition8A3}
Let $R$ be a serial relation on $U$ and $X\subseteq U$. Then $X\in D(U,R)$ iff $V_{R}(X)\subseteq X$.
\end{proposition}

\begin{proof}
$(\Rightarrow)$: By $X\in D(U,R)$, we know that $\overline{R}(X)=X=\underline{R}(X)$. Thus $\overline{R}(X)\subseteq X\subseteq\underline{R}(X)$. By Proposition \ref{proposition8A6}, we know that $V_{R}(X)\subseteq X$.

$(\Leftarrow)$: By Proposition \ref{proposition3A}, we know that $\underline{R}(X)\subseteq\overline{R}(X)$. By Proposition \ref{proposition8A6}, we know that $\overline{R}(X)\subseteq\underline{R}(X)$. Thus $\underline{R}(X)=\overline{R}(X)$. Again by Proposition \ref{proposition8A6}, we have that $\overline{R}(X)=X=\underline{R}(X)$. Then $X\in D(U,R)$.
$\Box$
\end{proof}

In fact, under the condition that $R$ is serial, we have the following proposition.

\begin{proposition}
\label{proposition17A}
Let $R$ be a serial relation on $U$. Then for any $X\subseteq U$, it follows that $V_{R}(X)=X$ iff $V_{R}(X)\subseteq X$.
\end{proposition}

\begin{proof}
$(\Rightarrow)$: It is straightforward.

$(\Leftarrow)$: We use the proof by contradiction. Suppose $V_{R}(X)\neq X$. Then $X-V_{R}(X)\neq\emptyset$. Without loss of generality, suppose $a\in X-V_{R}(X)$. Since $R$ is serial, there exists some $b\in U$ such that $b\in S_{R}(a)$. Since $S_{R}(a)\subseteq X$, $b\in X$. It is obvious $a\in P_{R}(b)$. Since $P_{R}(b)\subseteq V_{R}(X)$, $a\in V_{R}(X)$. It is contradictory.
$\Box$
\end{proof}

The following example indicates that the converse of the above proposition is not true.

\begin{example}
\label{example17A1}
Let $U=\{1,2,3\}$ and $R=\{(1,2),(3,1)\}$. Since $S_{R}(2)=\emptyset$, $R$ is not serial, but $\{X\subseteq U|V_{R}(X)\subseteq X\}=\{\emptyset,\{1,2,3\}\}=\{X\subseteq U|V_{R}(X)=X\}$.
\end{example}

Based on Propositions \ref{proposition17A} and \ref{proposition8A3}, we have the following proposition.

\begin{proposition}
\label{proposition8A4}
Let $R$ be a serial relation on $U$ and $X\subseteq U$. Then $X\in D(U,R)$ iff $V_{R}(X)=X$.
\end{proposition}

The following proposition presents a relationship between inner and outer definable set families.

\begin{proposition}
\label{proposition32B}
$X\in I(U,R)\Leftrightarrow-X\in O(U,R)$.
\end{proposition}

\begin{proof}
By Definition \ref{definition5} and Proposition \ref{proposition4}, we have that
$X\in I(U,R)\Leftrightarrow\underline{R}(X)=X\Leftrightarrow-\overline{R}(-X)=
X\Leftrightarrow\overline{R}(-X)=-X\Leftrightarrow-X\in O(U,R)$.
$\Box$
\end{proof}

According to the above proposition, we have the following proposition.

\begin{proposition}
\label{proposition32B11}
$X\in D(U,R)\Leftrightarrow-X\in D(U,R)$.
\end{proposition}

\begin{proof}
$X\in D(U,R)\Leftrightarrow(X\in I(U,R)\wedge X\in O(U,R))\Leftrightarrow(-X\in O(U,R)\wedge-X\in I(U,R))\Leftrightarrow-X\in D(U,R)$.
$\Box$
\end{proof}

Based on some above propositions, we present a sufficient condition for $I(U,R)=O(U,R)$.

\begin{proposition}
\label{proposition8B4}
Let $R$ be a serial relation on $U$. If for any $X\subseteq U$, $S_{R}(X)\subseteq X$ implies $P_{R}(X)\subseteq X$, $I(U,R)=O(U,R)$.
\end{proposition}

\begin{proof}
For any $X\in I(U,R)$, we know that $\underline{R}(X)=X$. Thus $X\subseteq\underline{R}(X)$. By Proposition \ref{proposition8A6}, we know that $S_{R}(X)\subseteq X$. Hence $P_{R}(X)\subseteq X$. Then $V_{R}(X)\subseteq X$. By Proposition \ref{proposition8A3}, we know that $X\in D(U,R)\subseteq O(U,R)$. Therefore $I(U,R)\subseteq O(U,R)$. For any $Y\in O(U,R)$, by Proposition \ref{proposition32B}, we know that $-Y\in I(U,R)$. Thus $\underline{R}(-Y)=-Y$. Hence $-Y\subseteq\underline{R}(-Y)$. By Proposition \ref{proposition8A6}, we know that $S_{R}(-Y)\subseteq -Y$. Then $P_{R}(-Y)\subseteq -Y$. Thus $V_{R}(-Y)\subseteq -Y$. By Proposition \ref{proposition8A3}, we know that $-Y\in D(U,R)$. By Proposition \ref{proposition32B11}, we know that $Y\in D(U,R)\subseteq I(U,R)$. Therefore $O(U,R)\subseteq I(U,R)$. Then $I(U,R)=O(U,R)$.
$\Box$
\end{proof}

The converse of the above proposition is not true. To illustrate this, let us see an example.

\begin{example}
\label{example39B}
Let $U=\{1,2,3\}$ and $R=\{(1,1),(1,2),(2,1),(3,2)\}$. Then $S_{R}(\{1,2\})\\
=\{1,2\}\subseteq\{1,2\}$ and $P_{R}(\{1,2\})=\{1,2,3\}\nsubseteq\{1,2\}$, but $I(U,R)=\{\emptyset,\{1,2,3\}\}=O(U,R)$.
\end{example}

For applying Proposition \ref{proposition8B4}, we first present the following lemma.

\begin{lemma}
\label{lemma8B5}
Let $R$ be a symmetric relation on $U$. Then for any $X\subseteq U$, $S_{R}(X)=P_{R}(X)$.
\end{lemma}

\begin{proof}
For any $a\in U$ and any $b\in S_{R}(a)$, we know that $aRb$. Since $R$ is symmetric, $bRa$. Thus $b\in P_{R}(a)$. Hence $S_{R}(a)\subseteq P_{R}(a)$. Similarly, $P_{R}(a)\subseteq S_{R}(a)$. Then $S_{R}(X)=\cup_{x\in X}S_{R}(x)=\cup_{x\in X}P_{R}(x)=P_{R}(X)$.
$\Box$
\end{proof}

As a application of Proposition \ref{proposition8B4}, we have the following proposition.

\begin{proposition}
\label{proposition8B6}
Let $R$ be a serial and symmetric relation on $U$. Then $I(U,R)=O(U,R)$.
\end{proposition}

\begin{proof}
It follows from Lemma \ref{lemma8B5} and Proposition \ref{proposition8B4}.
$\Box$
\end{proof}

\section{Simplification of equivalent closures under certain conditions}
\label{Simplification of equivalent closures under certain conditions}

In this section, we simplify the expression of equivalent closures under certain conditions and prove that serial relations satisfy these conditions. By Theorem \ref{theorem29}, we obtain the computational formula of equivalent closures.

\begin{corollary}
\label{corollary39A3}
Let $R\subseteq A\times A$ and $A\neq \emptyset$. Then $e(R)=I_{A}\cup(R\cup R^{-1})\cup(R\cup R^{-1})^{2}\cup\cdots$.
\end{corollary}

For dealing with the power of relations, we present the following proposition.

\begin{proposition}
\label{proposition28A}
Let $R\subseteq A\times A$, $A\neq \emptyset$, $k\in N^{+}$ and $k\geq2$. Then $(x,y)\in R^{k}$ iff there exist $x_{1},x_{2},\cdots,x_{k-1}\in A$ such that $(x,x_{1}),(x_{1},x_{2}),\cdots,(x_{k-2},x_{k-1}),(x_{k-1},y)\in R$.
\end{proposition}

\begin{proof}
$(\Rightarrow)$: We prove this assertion using induction on $k$. If $k=2$, this assertion follows easily from Definitions \ref{definition27} and \ref{definition28}. Assume this assertion is true for $k\leq t-1$. Now assume $k=t$. By Definition \ref{definition28}, we have that $(x,y)\in R^{t-1}\circ R$. By Definition \ref{definition27}, we know that there exists some $x_{1}$ such that $(x,x_{1})\in R$ and $(x_{1},y)\in R^{t-1}$. By the assumption of the induction, we know  that there exist $x_{2},x_{3},\cdots,x_{t-1}\in A$ such that $(x_{1},x_{2}),(x_{2},x_{3}),\cdots,(x_{t-2},x_{t-1}),(x_{t-1},y)\in R$. Thus $(x,x_{1}),(x_{1},x_{2}),\cdots,(x_{t-2},\\
x_{t-1}),(x_{t-1},y)\in R$.

$(\Leftarrow)$: We prove this assertion using induction on $k$. If $k=2$, this assertion follows easily from Definitions \ref{definition27} and \ref{definition28}. Assume this assertion is true for $k\leq t-1$. Now assume $k=t$. By the assumption of the induction, we know  that $(x_{1},y)\in R^{t-1}$. Again by $(x,x_{1})\in R$, Definitions \ref{definition27} and \ref{definition28}, we have that $(x,y)\in R^{t-1}\circ R=R^{t}$.
$\Box$
\end{proof}

The following two lemmas present two properties of symmetric relations.

\begin{lemma}
\label{lemma28E}
Let $A\neq \emptyset$, $R\subseteq A\times A$, $R$ be symmetric, $k\in N^{+}$ and $x\in A$. Then $(x,x)\in R^{k}$ iff $(x,x)\in R^{2}$.
\end{lemma}

\begin{proof}
$(\Rightarrow)$: If $k=1$, this assertion is obviously true. Below we suppose $k\geq2$. By Proposition \ref{proposition28A}, we know that there exist $x_{1},x_{2},\cdots,x_{k-1}\in A$ such that $(x,x_{1}),(x_{1},x_{2}),\\
\cdots,(x_{k-2},x_{k-1}),(x_{k-1},x)\in R$. Since $R$ is symmetric, $(x_{1},x)\in R$. Thus $(x,x)\in R^{2}$.

$(\Leftarrow)$: It is straightforward.
$\Box$
\end{proof}

\begin{lemma}
\label{lemma28D}
Let $A\neq \emptyset$, $R\subseteq A\times A$, $R$ be symmetric and $x\in A$. Then $(x,x)\in R^{2}$ iff $S_{R}(x)\neq\emptyset$.
\end{lemma}

\begin{proof}
$(\Rightarrow)$: It is obvious there exists some $y\in A$ such that $(x,y)\in R$. Thus $y\in S_{R}(x)\neq\emptyset$.

$(\Leftarrow)$: Without loss of generality, suppose $z\in S_{R}(x)$. Thus $(x,z)\in R$. Since $R$ is symmetric, $(z,x)\in R$. Hence $(x,x)\in R^{2}$.
$\Box$
\end{proof}

The following proposition presents a simple property of successor neighborhoods and predecessor neighborhoods.

\begin{proposition}
\label{proposition39A1}
Let $R_{1},R_{2}$ be two relations on $U$ and $x\in U$. If $R_{1}\subseteq R_{2}$, $S_{R_{1}}(x)\subseteq S_{R_{1}}(x)$ and $P_{R_{1}}(x)\subseteq P_{R_{1}}(x)$.
\end{proposition}

\begin{proof}
For any $y\in S_{R_{1}}(x)$, we have that $(x,y)\in R_{1}$. Thus $(x,y)\in R_{2}$. Hence $y\in S_{R_{2}}(x)$. Therefore $S_{R_{1}}(x)\subseteq S_{R_{2}}(x)$. In the same way, we have that $P_{R_{1}}(x)\subseteq P_{R_{2}}(x)$.
$\Box$
\end{proof}

By the above proposition, we present an expression of the successor neighborhoods of the union of two relations.

\begin{lemma}
\label{lemma28F}
Let $A\neq \emptyset$, $R_{1}$ and $R_{2}$ be two relations on $A$ and $x\in A$. Then $S_{R_{1}\cup R_{2}}(x)=S_{R_{1}}(x)\cup S_{R_{2}}(x)$.
\end{lemma}

\begin{proof}
$(\subseteq)$: For any $z\in S_{R_{1}\cup R_{2}}(x)$, we have that $(x,z)\in R_{1}\cup R_{2}$. If $(x,z)\in R_{1}$, we have that $z\in S_{R_{1}}(x)$. If $(x,z)\in R_{2}$, we have that $z\in S_{R_{1}}(x)$. Hence $z\in S_{R_{1}}(x)\cup S_{R_{2}}(x)$. Therefore $S_{R_{1}\cup R_{2}}(x)\subseteq S_{R_{1}}(x)\cup S_{R_{2}}(x)$.

$(\supseteq)$: It follows from Proposition \ref{proposition39A1}.
$\Box$
\end{proof}

Based on some above results, we present a necessary and sufficient condition for $e(R)=t(s(R))$.

\begin{theorem}
\label{theorem39A7}
Let $R\subseteq A\times A$ and $A\neq \emptyset$. Then $e(R)=ts(R)$ iff for any $x\in A$, it follows that $S_{R}(x)\cup S_{R^{-1}}(x)\neq\emptyset$.
\end{theorem}

\begin{proof}

By Theorem \ref{theorem29}, Corollary \ref{corollary39A3}, Lemmas \ref{lemma28E}, \ref{lemma28D} and \ref{lemma28F}, we have that $e(R)=ts(R)\Leftrightarrow I_{A}\cup(R\cup R^{-1})\cup(R\cup R^{-1})^{2}\cup\cdots=(R\cup R^{-1})\cup(R\cup R^{-1})^{2}\cup\cdots\Leftrightarrow I_{A}\subseteq(R\cup R^{-1})\cup(R\cup R^{-1})^{2}\cup\cdots\Leftrightarrow\forall x\in A\exists k\in N^{+}((x,x)\in (R\cup R^{-1})^{k})\Leftrightarrow\forall x\in A((x,x)\in (R\cup R^{-1})^{2})\Leftrightarrow\forall x\in A(S_{R\cup R^{-1}}(x)\neq\emptyset)\Leftrightarrow\forall x\in A(S_{R}(x)\cup S_{R^{-1}}(x)\neq\emptyset)$.
$\Box$

\end{proof}

Applying the above theorem to a serial relation, we have the following corollary.

\begin{corollary}
\label{corollary39A5}
Let $R\subseteq A\times A$ and $A\neq \emptyset$. If $R$ is serial, $e(R)=ts(R)=(R\cup R^{-1})\cup(R\cup R^{-1})^{2}\cup\cdots$.
\end{corollary}

\begin{proof}
Since $R$ is serial, for any $x\in A$, it follows that $S_{R}(x)\neq\emptyset$. Thus $S_{R}(x)\cup S_{R^{-1}}(x)\neq\emptyset$. By Theorems \ref{theorem39A7} and \ref{theorem29}, we have that $e(R)=ts(R)=(R\cup R^{-1})\cup(R\cup R^{-1})^{2}\cup\cdots$.
\end{proof}

\section{A necessary and sufficient condition for two relations to induce the same definable set family}
\label{A necessary and sufficient condition for two relations to induce the same definable set family}

In this section, based on the above sections, we present a necessary and sufficient condition for two relations to induce the same definable set family.
The following lemma presents a property of definable set families.

\begin{lemma}
\label{lemma39A9}
Let $R$ be a serial relation on $U$, $x\in X\subseteq U$ and $(x,y)\in R\cup R^{-1}$. If $X\in D(U,R)$, $y\in X$.
\end{lemma}

\begin{proof}
It is obvious $(x,y)\in R$ or $(x,y)\in R^{-1}$. Then $y\in S_{R}(x)$ or $y\in P_{R}(x)$. Thus $y\in V_{R}(X)$. By Proposition \ref{proposition8A3}, we know that $V_{R}(X)\subseteq X$. Hence $y\in X$.
$\Box$
\end{proof}

The following lemma is an extension of the above lemma.

\begin{lemma}
\label{lemma39A}
Let $R$ be a serial relation on $U$, $x\in X\subseteq U$, $k$ be a nonnegative integer and $(x,x_{1}),(x_{1},x_{2}),\cdots,(x_{k-1},x_{k}),(x_{k},y)\in R\cup R^{-1}$. If $X\in D(U,R)$, $y\in X$.
\end{lemma}

\begin{proof}
We prove this assertion using induction on $k$. The assertion under the condition of $k=0$ has been proved in Lemma \ref{lemma39A9}. Assume this assertion is true for $k\leq t-1$. Now assume $k=t$. By the assumption of the induction, we know that $x_{k}\in X$. Thus by Lemma \ref{lemma39A9}, we have that $y\in X$.
$\Box$
\end{proof}

Now we can prove one of the main results in this paper, which indicates that a serial relation and its equivalent closure induce the same definable set family.

\begin{theorem}
\label{theorem40B}
Let $R$ be a serial relation on $U$. Then $D(U,R)=D(U,ts(R))$.
\end{theorem}

\begin{proof}
$(\subseteq)$: Let $x\in X\in D(U,R)$. By Corollary \ref{corollary39A5}, we know that for any $y\in S_{ts(R)}(x)$, it follows that $(x,y)\in (R\cup R^{-1})\cup(R\cup R^{-1})^{2}\cup\cdots$. Thus $(x,y)\in R\cup R^{-1}$ or $(x,y)\in(R\cup R^{-1})^{k}$, where $k\in N^{+}\wedge k\geq2$. If $(x,y)\in R\cup R^{-1}$, by Lemma \ref{lemma39A9}, we have that $y\in X$. If $(x,y)\in(R\cup R^{-1})^{k}$, where $k\in N^{+}\wedge k\geq2$, by Proposition \ref{proposition28A}, we know that there exist $x_{1},x_{2},\cdots,x_{k-1}\in X$ such that $(x,x_{1}),(x_{1},x_{2}),\cdots,(x_{k-2},x_{k-1}),(x_{k-1},y)\in R\cup R^{-1}$. By Lemma \ref{lemma39A}, we know that $y\in X$. Therefore $S_{ts(R)}(x)\subseteq X$. Since $P_{ts(R)}(x)=S_{ts(R)}(x)$, $P_{ts(R)}(x)\subseteq X$. Thus $V_{ts(R)}(X)\subseteq X$. By Proposition \ref{proposition8A3}, we know that $X\in D(U,ts(R))$.

$(\supseteq)$: By Proposition \ref{proposition8A3}, we know that for any $X\in D(U,ts(R))$ and any $x\in X$, it follows that $S_{ts(R)}(x)=P_{ts(R)}(x)\subseteq X$. By (2) of Definition \ref{definition35} and Proposition \ref{proposition39A1}, we have that $S_{R}(x)\subseteq S_{ts(R)}(x)$ and $P_{R}(x)\subseteq P_{ts(R)}(x)$. Thus $S_{R}(x)\subseteq X$ and $P_{R}(x)\subseteq X$. Hence $V_{R}(X)\subseteq X$. By Proposition \ref{proposition8A3}, we know that $X\in D(U,R)$.
$\Box$
\end{proof}

The following example indicates that if $R$ is not serial, $D(U,R)=D(U,ts(R))$ is incorrect.

\begin{example}
\label{example40c}
Let $U=\{1,2,3\}$ and $R=\{(1,2),(3,3)\}$. Then $ts(R)=\{(1,1),(2,2),(3,\\
3),(1,2),(2,1)\}$ and $D(U,R)=\emptyset\neq\{\emptyset,\{3\},\{1,2\},\{1,2,3\}\}=D(U,ts(R))$.
\end{example}

For presenting a necessary and sufficient condition for two relations to induce the same definable set family, we need to prove the following proposition.

\begin{proposition}
\label{proposition40B1}
Let $R_{1},R_{2}$ be two equivalence relations on $U$. Then $D(U,R_{1})=D(U,R_{2})$ iff $R_{1}=R_{2}$.
\end{proposition}

\begin{proof}
$(\Rightarrow)$: We use the proof by contradiction. Suppose $R_{1}\neq R_{2}$. By Theorem \ref{theoremA7}, we have that $U/R_{1}\neq U/R_{2}$. Thus $U/R_{1}-U/R_{2}\neq\emptyset$. Without loss of generality, suppose $K\in U/R_{1}-U/R_{2}$. Then there exists some $L\in U/R_{2}$ such that $K\subset L$ or for any $J\in U/R_{2}$, it follows that $K\nsubseteq J$. If $K\subset L$, we have that $\underline{R_{2}}(K)=\emptyset$. Thus $K\in D(U,R_{1})-D(U,R_{2})$. Hence $D(U,R_{1})\neq D(U,R_{2})$. If for any $J\in U/R_{2}$, it follows that $K\nsubseteq J$, there exists some $A\subseteq U/R_{2}$ such that $K\subseteq\cup A$, $|A|\geq2$ and for any $B\subset A$, it follows that $K\nsubseteq\cup B$. For any $I\in A$, it is obvious $K\cap I\neq\emptyset$ and $K-I\neq\emptyset$. Thus $K\subseteq\overline{R_{1}}(I)$. Hence $\emptyset\neq K-I\subseteq\overline{R_{1}}(I)-I$. Then $\overline{R_{1}}(I)\neq I$. Therefore $I\in D(U,R_{2})-D(U,R_{1})$. Then $D(U,R_{1})\neq D(U,R_{2})$.

$(\Leftarrow)$: It is straightforward.
$\Box$
\end{proof}

Finally, based on some above results, we present a necessary and sufficient condition for two relations to induce the same definable set family.

\begin{theorem}
\label{theorem40B2}
$D(U,R_{1})=D(U,R_{2})$ iff both $R_{1}$ and $R_{2}$ are not serial, or both $R_{1}$ and $R_{2}$ are serial and $ts(R_{1})=ts(R_{2})$.
\end{theorem}

\begin{proof}
$(\Rightarrow)$: It is obvious $D(U,R_{1})=D(U,R_{2})=\emptyset$ or $D(U,R_{1})=D(U,R_{2})\neq\emptyset$. If $D(U,R_{1})=D(U,R_{2})=\emptyset$, by Theorem \ref{theorem32A}, we have that both $R_{1}$ and $R_{2}$ are not serial. If $D(U,R_{1})=D(U,R_{2})\neq\emptyset$, by Theorem \ref{theorem32A}, we have that both $R_{1}$ and $R_{2}$ are serial. Thus by Theorem \ref{theorem40B}, we have that $D(U,ts(R_{1}))=D(U,R_{1})=D(U,R_{2})=D(U,ts(R_{2}))$. By Proposition \ref{proposition40B1} and Corollary \ref{corollary39A5}, we have that $ts(R_{1})=ts(R_{2})$.

$(\Leftarrow)$: If both $R_{1}$ and $R_{2}$ are not serial, by Theorem \ref{theorem32A}, we have that $D(U,R_{1})=\emptyset=D(U,R_{2})$. If both $R_{1}$ and $R_{2}$ are serial and $ts(R_{1})=ts(R_{2})$, by Theorem \ref{theorem40B}, we have that $D(U,R_{1})=D(U,ts(R_{1}))=D(U,ts(R_{2}))=D(U,R_{2})$.
$\Box$
\end{proof}

\section{Conclusions }
\label{S:Conclusions}
In this paper, we studied under what condition two relations induce the same definable set family. First, we introduced some definitions and results of relation closures. Secondly, we investigated some fundamental properties of definable sets. Thirdly, we simplified the expression of equivalent closures under certain conditions. Finally, based on the research of equivalent closures, we presented a necessary and sufficient condition for two relations to induce the same definable set family. Relation based rough sets are only one type of generalizing rough sets. There are some issues related to definable sets unsolved in various types of generalizing rough sets, which will be investigated in our future works.

\section*{Acknowledgments}
This work is in part supported by the National Natural Science Foundation of China under Grant Nos. 61170128, 61379049 and 61379089, the Natural Science Foundation of Fujian Province, China under Grant No. 2012J01294, the Fujian Province Foundation of Higher Education under Grant No. JK2012028, and the Postgraduate Education Innovation Base for Computer Application Technology, Signal and Information Processing of Fujian Province (No. [2008]114, High Education of Fujian).


\end{document}